%% file: VE.tex
\documentclass{ecai}
\usepackage{times}
\usepackage{graphicx}
\usepackage{latexsym}

\ecaisubmission   

\usepackage{amsmath}
\usepackage{amssymb}
\usepackage{amsthm}
\usepackage{framed}
\usepackage{graphicx}
\usepackage{color}
\usepackage{algorithm}
\usepackage{algpseudocode}
\usepackage{url}
\usepackage{caption}

\usepackage{mathtools,stackengine}
\usepackage[mathscr]{euscript}
\usepackage{subfigure}
\usepackage{multirow}
\usepackage{accents}

\usepackage{pifont}
\newcommand{\cmark}{\ding{51}}%
\newcommand{\xmark}{\ding{55}}%

\def\comment(#1){{\textcolor{blue}{\textbf{#1}}}}
\def\eql(#1,#2){{#1\!\!=\!#2}}

\input{stdmacros.tex}

\def\numpy{{\textsc NumPy}}

\newcommand\scalemath[2]{\scalebox{#1}{\mbox{\ensuremath{\displaystyle #2}}}}

\def\Sum{\sum^{\raisebox{-5pt}{$\scriptscriptstyle =$}}}
\def\eql(#1,#2){{#1\!\!=\!\!#2}}

\def\p{{P}}
\def\P{{P^\star}}
\def\NC{{\eta}}
\def\FF{{\cal F}}
\def\GG{{\cal G}}
\def\HH{{\cal H}}

\def\facs(#1){\FF_{#1}}
\def\facsp(#1){\accentset{\frown}{\FF}_{#1}}
\def\facsn(#1){\accentset{\smile}{\FF}_{#1}}

\def\sep{\mathtt{sep}}

\def\fvars{\mathtt{fvars}}

\def\fsum{\mathtt{fsum}}

\def\mes(#1,#2){{\mathscr{M}(#1,#2)}}

\begin{document}

\title{An Advance on Variable Elimination with Applications to Tensor-Based Computation\footnote{Will appear in proceedings 
of the European Conference on Artificial Intelligence (ECAI), Spain 2020.}}

\author{Adnan Darwiche\institute{University of California, Los Angeles,
email: darwiche@cs.ucla.edu}}

\maketitle
\bibliographystyle{ecai}

\begin{abstract}
We present new results on the classical algorithm of variable elimination, which underlies many algorithms including for probabilistic
inference. The results relate to exploiting functional dependencies, allowing one to perform inference efficiently on models 
that have very large treewidth. The highlight of the advance is that it works with standard (dense) factors, without the need for 
sparse factors or techniques based on knowledge compilation that are commonly utilized. 
This is significant as it permits a direct implementation of the improved variable elimination algorithm using tensors and their operations, 
leading to extremely efficient implementations especially when learning model parameters. 
We illustrate the efficacy of our proposed algorithm by compiling Bayesian network queries into tensor graphs and then 
learning their parameters from labeled data using a standard tool for tensor computation.
\end{abstract}

\section{Introduction}

The work reported in this paper is motivated by an interest in model-based supervised learning, in contrast
to model-free supervised learning that currently underlies most applications of neural networks. We briefly 
discuss this subject first to put the proposed work in context.

Supervised learning has become very influential recently and stands behind most real-world applications of AI. In supervised learning,
one learns a {\em function} from labeled data, a practice that is now dominated by the use of neural networks to represent such functions;
see~\cite{Goodfellow-et-al-2016,HintonOT06,BengioLPL06,RanzatoPCL06}.
Supervised learning can be applied in other contexts as well, such as causal models in the form of Bayesian networks~\cite{Pearl88b,pearl00b,pearl18}. 
In particular,  for each query on the causal model, one can {\em compile} an Arithmetic Circuit (AC) that maps evidence (inputs) to the posterior
probability of interest (output)~\cite{Darwiche03,Darwiche09}. AC parameters, which correspond to Bayesian network parameters, 
can then be learned from labeled data using gradient descent. Hence, like a neural network, the AC is a circuit that computes a 
function whose parameters can be learned from labeled data.

The use of ACs in this fashion can be viewed as {\em model-based} supervised learning, in contrast to {\em model-free} supervised learning 
using neural networks. Model-based supervised learning is attractive since the AC can integrate the {\em background knowledge} 
embecded in its underlying causal model. This has a number of advantages, which include a reduced reliance on data, improved
robustness and the ability to provide data-independent guarantees on the learned function. One important type of background 
knowledge is functional dependencies between variables and their direct causes in a model (a special case of what is known
as {\em determinism}). Not only can this type of knowledge
significantly reduce the reliance on data, but it can also significantly improve the complexity of inference. In fact, substantial efforts have been
dedicated to exploiting determinism in probabilistic inference,
particularly the compilation of Bayesian networks into ACs~\cite{Darwiche03,ChaviraDJ06,chavira.Darwiche.Aij.2008}, 
which is necessary for efficient inference on dense models.

There are two main approaches for exploiting functional dependencies. The first is based on the classical algorithm of variable
elimination (VE), which underlies algorithms for probabilistic inference including the {\em jointree} algorithm~\cite{zhangJAIR96a,dechterUAI96,JT_ALG}. 
VE represents a model using {\em factors,} which are tables or multi-dimensional arrays.
It then performs inference using a few and simple {\em factor operations.} 
Exploiting functional dependencies within VE requires {\em sparse factors}; see, e.g.,~\cite{Jensen90,LarkinD03}. 
The second approach for exploiting functional dependencies reduces probabilistic inference 
to {\em weighted model counting} on a propositional formula that encodes the model, including its functional 
dependencies. It then {\em compiles} the formula into a circuit that is tractable for model counting; 
see, e.g.,~\cite{Darwiche02,chavira.Darwiche.Aij.2008}. This approach is in common use today given the
efficacy of knowledge compilers.

Our main contribution is a new approach for exploiting functional dependencies in VE
that works with {\em standard (dense) factors.} This is significant for the following
reason. We wish to map probabilistic inference, particularly the learning of parameters, into a {\em tensor computation}
to exploit the vast progress on tensor-based technology and be on par with approaches that aggressively exploit this technology. 
Tensors are multi-dimensional arrays
whose operations are heavily optimized and can be extremely fast, even on CPU-based platforms like modern
laptops (let alone GPUs). A tensor computation takes the form of a {\em tensor graph} with nodes representing tensor operations.
Factors map directly to tensors and sparse factors to sparse tensors. However, sparse tensors
have limited support in state of the art tools, which prohibits an implementation of VE using sparse
tensors.\footnote{For example, in TensorFlow, a sparse tensor can only be multiplied by a dense tensor, which
rules out the operation of (sparse) factor multiplication that is essential for sparse VE; 
see \url{https://www.tensorflow.org/api_docs/python/tf/sparse/SparseTensor}.} 
Knowledge compilation approaches produce circuits that cast into {\em scalar
tensor graphs,} which are less effective than general tensor graphs as they are less amenable to parallelization.
Moreover, while our approach needs to know that there is a functional
dependency between variables it does not require the specific dependency (the specific numbers). 
Hence, it can be used to speed up inference even when the model parameters are unknown which can be
critical when learning model parameters from data. 
Neither of the previous approaches can exploit this kind of abstract information.

VE is based on two theorems that license factor operations. We add two new theorems
that license more operations in the presence of functional dependences. This leads to a standard VE
algorithm except with a significantly improved complexity and computation that maps directly to a tensor graph.
We present experimental results for inference and learning that show promise of the proposed algorithm. 

We start in Section~\ref{sec:factors} by discussing factors, their operations and the VE algorithm including its
underlying theorems. We also present our new VE theorems in this section. We then propose a new VE algorithm 
in Section~\ref{sec:alg} that exploits functional dependencies. 
We show how the proposed algorithm maps to tensor graphs and why this matter in Section~\ref{sec:tensors}.  
We follow by case studies in Section~\ref{sec:studies} that illustrate the algorithm's 
performance in the context of model-based supervised learning. We finally close with some remarks in 
Section~\ref{sec:conclusion}.

\section{The Fundamentals: Factors \& Operations}
\label{sec:factors}

The VE algorithm is based on applying operations to factors.

A {\em factor} for discrete variables \(\X\) is a function that maps each instantiation \(\x\) of 
variables \(\X\) into a number. The following are two factors over binary variables
\(A, B, C\) and ternary variable \(D\):
\begin{center}
\scalebox{0.7}{
\(
\begin{array}{cc|c}
A & D & f(AD)\\ \hline
0 & 0 & 0.2 \\
0 & 1 & 0.3 \\
0 & 2 & 0.6 \\
1 & 0 & 0.9 \\
1 & 1 & 0.6 \\
1 & 2 & 0.1 
\end{array}
\)
\hspace{5mm}
\(
\begin{array}{ccc|c}
A & B & C & g(ABC)\\ \hline
0 & 0 & 0 & 1.0 \\
0 & 0 & 1 & 0.0 \\
0 & 1 & 0 & 0.0 \\
0 & 1 & 1 & 1.0 \\
1 & 0 & 0 & 0.2 \\
1 & 0 & 1 & 0.8 \\
1 & 1 & 0 & 0.5 \\
1 & 1 & 1 & 0.5 
\end{array}
\)
}
\end{center}
Factors can be represented as multi-dimensional arrays and are now commonly
referred to as {\em tensors} (factor variables corresponds to array/tensor dimensions). 
One needs three factor operations to implement the VE algorithm: multiplication, sum-out and normalization.

The {\em product} of factors \(f(\X)\) and \(g(\Y)\) is another factor \(h(\Z)\),
where \(\Z = \X \cup \Y\) and \(h(\z) = f(\x)g(\y)\) for the unique instantiations \(\x\) and \(\y\)
that are compatible with instantiation~\(\z\). 
{\em Summing-out} variables \(\Y \subseteq \X\) from factor \(f(\X)\) yields another factor \(g(\Z)\),
where \(\Z= \X \setminus \Y\) and \(g(\z) = \sum_\y f(\y\z)\). 
We use \(\sum_\Y f\) to denote the resulting factor \(g\). We also use \(\Sum_\Z f\) which reads:
sum out all variables from factor \(f\) except for variables \(\Z\).
{\em Normalizing} factor \(f(\X)\) yields another factor \(g(\X)\) where \(g(\x) = f(\x)/\sum_\x f(\x)\). 
We use \(\NC f\) to denote the normalization of factor \(f\).
 
A {\em Bayesian Network (BN)} is specified by a directed acyclic graph (DAG) and a set of 
factors. In particular, for each node \(X\) and its parents \(\U\), we need a factor \(f_X\) over variables \(X\U\). 
The value \(f_X(x\u)\) represents the conditional probability \(\p(x|\u)\) and the factor \(f_X\) is called 
a {\em Conditional Probability Table (CPT).} The {\em joint distribution} specified by a Bayesian network 
is simply the product of its CPTs.

The Bayesian network in Figure~\ref{fig:bn-vw}  has five CPTs \(f_A(A)\), \(f_B(AB)\), \(f_C(AC)\), \(f_D(BCD)\) and 
\(f_E(CE)\). The network joint distribution is the product of these factors \(\pr(ABCDE) = f_A f_B f_C f_D f_E\).

Evidence on variable \(X\) is captured by a factor \(\lambda_X(X)\) called an 
{\em evidence indicator.} 
{\em Hard evidence} fixes a value \(x\) giving \(\lambda_X(x)=1\) 
and \(\lambda_X(x^\star)=0\) for  \(x^\star \neq x\). For {\em soft evidence,} 
\(\lambda_X(x)\) is the {\em likelihood} of \(x\)~\cite{Pearl88b}. 
The {\em posterior distribution} of a Bayesian network is the normalized product of its
CPTs and evidence indicators.

An expression constructed by applying operations to factors will be called an {\em f-expression.}
Suppose we have evidence on variables \(A\) and \(E\) in Figure~\ref{fig:bn-vw}. 
The posterior on variable \(D\) is obtained by evaluating the following f-expression:
\[
\P(D) = \NC \sum_{ABCE} \lambda_A \lambda_E f_A f_B f_C f_D f_E.
\]
The VE algorithm factors f-expressions so they are evaluated more 
efficiently~\cite{zhangJAIR96a,dechterUAI96} and is based on two theorems; see, e.g., \cite[Chapter 6]{Darwiche09}.
The first theorem says that the order in which variables are summed out does not matter.

\begin{thm}\label{theo:ve0}
\(\sum_{\X\Y} f = \sum_\X \sum_\Y f = \sum_\Y \sum_\X f\).
\end{thm}

The second theorem allows us to reduce the size of factors involved in a multiplication operation.

\begin{thm}\label{theo:ve1}
If variables \(\X\) appear in factor \(f\) but not in factor \(g\), then \(\sum_\X f \cdot g = g \sum_\X f\).
\end{thm}
Factor \(\sum_\X f\) is exponentially smaller than factor \(f\) so Theorem~\ref{theo:ve1} allows us to
evaluate the f-expression \(\sum_\X f \cdot g\) much more efficiently.  

Consider the f-expression \(\sum_{ABDE} f(ACE) f(BCD)\). A direct evaluation
multiplies the two factors to yield \(f(ABCDE)\) then sums out variables \(ABDE\). Using
Theorem~\ref{theo:ve0}, we can arrange the expression into \(\sum_{AE} \sum_{BD} f(ACE) f(BCD)\)
and using Theorem~\ref{theo:ve1} into \(\sum_{AE} f(ACE) \sum_{BD} f(BCD)\), which is
more efficient to evaluate. 

Using an appropriate order for summing out (eliminating) variables, 
Theorems~\ref{theo:ve0} and~\ref{theo:ve1} allow one to compute the posterior on any variable 
in a Bayesian network in \(O(n \exp(w))\) time and space. Here, \(n\) is the number of 
network variables and \(w\) is the network treewidth (a graph-theoretic  measure of the 
network connectivity).

This works well for sparse networks that have a small treewidth, but is problematic for dense 
networks like the ones we will look at in Section~\ref{sec:studies}. We present next two
new results that allow us to sometimes significantly improve this computational complexity,
by exploiting functional relationships between variables and their direct causes. 

While we will focus on exploiting functional dependencies in Bayesian networks, our results are more broadly applicable
since the VE algorithm can be utilized in many other domains
including symbolic reasoning and constraint processing~\cite{RinaCSPBook}. VE can
also be used to {\em contract} tensor networks which have been receiving increased attention.
A {\em tensor network} is a set of factors in which a variable appears in at most
two factors. Contracting a tensor network is the problem of summing
out all variables that appear in two factors; see, e.g.,~\cite{TensorOrder,ContractionTrees}. 
The VE algorithm can also be used to evaluate Einstein summations which are in common use today and
implemented in many tools including \numpy.\footnote{\url{https://numpy.org/}}

\subsection{Functional CPTs}

Consider variable \(X\) that has parents \(\U\) in a Bayesian network and let factor \(f_X(X\U)\)
be its conditional probability table (CPT).\footnote{Since \(\sum_x \p(x|\u) = 1\) the CPT satisfies
\(\sum_x f_X(x\u) = 1\) for every \(\u\).}
If \(f(x\u) \in \{0,1\}\) for all instantiations \(x\) and \(\u\), the CPT is said to be {\em functional} 
as it specifies a function that maps parent instantiation \(\u\) into the unique 
value \(x\) satisfying \(f_X(x\u)=1\). The following CPTs are functional:
\begin{center}
\scalebox{0.7}{
\(
\begin{array}{cc|c}
X & Y & f_Y(XY)\\ \hline
x_0 & y_0 & 0 \\
x_0 & y_1 & 1 \\
x_1 & y_0 & 1 \\
x_1 & y_1 & 0 
\end{array}
\)
\hspace{5mm}
\(
\begin{array}{cc|c}
A & B & f_B(AB)\\ \hline
a_0 & b_0 & 0 \\
a_0 & b_1 & 1 \\
a_0 & b_2 & 0 \\
a_1 & b_0 & 0 \\
a_1 & b_1 & 0 \\
a_1 & b_2 & 1 
\end{array}
\)
}
\end{center}
The first specifies the function \(x_0 \mapsto y_1\), \(x_1 \mapsto y_0\).
The second specifies the function \(a_0 \mapsto b_1\), \(a_1 \mapsto b_2\).
Functional dependencies encode a common type of background knowledge (examples 
in Section~\ref{sec:studies}). They are a special type of {\em determinism,} which 
generally refers to the presence of zero parameters in a CPT. A CPT that has zero
parameters is not necessarily a functional CPT.

We will next present two results that empower the VE algorithm in the presence of functional CPTs.
The results allow us to factor f-expressions beyond what is permitted by 
Theorems~\ref{theo:ve0} and~\ref{theo:ve1}, leading to significant reduction in complexity. 
{\em The results do not depend on the identity of a functional CPT, only that it is functional.}
This is significant when learning model parameters from data.

To state these results, we will use \(\FF\), \(\GG\) and \(\HH\) to denote sets of factors. 
Depending on the context, a set of factors \(\FF\) may be treated as one factor obtained by 
multiplying members of the set \(\prod_{f \in \FF} f\).

The first result says the following. If a functional CPT for variable \(X\) appears in both parts of a product, then
variable \(X\) can be summed out from one part without changing the value of the product.

\begin{thm}\label{theo:ve2}
Consider a functional CPT \(f\) for variable \(X\). If  \(f \in \GG\) and \(f \in \HH\), then
\(\GG \cdot \HH = \GG \sum_X \HH\). 
\end{thm}

\begin{proof}
Suppose CPT \(f\) is over variables \(X\U\).
Let \(h(\X)\) and \(g(\Y)\) be the factors corresponding to \(\HH\) and \(\GG\), respectively.
Let \(\Z = \X \cup \Y\) and \(\X^\star = \X \setminus \{X\}\). Then variables \(X\U\) must belong
to \(\X\), \(\Y\) and \(\Z\), and parents \(\U\) must belong to \(\X^\star\). 
Let \(e_l = \GG \cdot \HH\) and \(e_r = \GG \sum_X \HH\).
We want to show \(e_l(\z) = e_r(\z)\) for every instantiation \(\z\).

Consider an instantiation \(\z\) and let \(\u\), \(\x^\star\), \(\x\) and \(\y\) be the instantiations of \(\U\),  \(\X^\star\), \(\X\) and \(\Y\) in \(\z\).
Then \(e_l(\z) = g(\y)h(\x)\) and \(e_r(\z) = g(\y) \sum_x h(x\x^\star)\). 
Since CPT \(f\) is functional, \(f(x\u) \in \{0,1\}\) for any \(x\) and there is a unique \(x\), call it \(x_\u\),
such that \(f(x\u) = 1\).

If \(f(x\u)=0\), then \(h(x\x^\star)=0\) since \(f \in \HH\), leading to
\[
\scalemath{0.90}{
e_r(\z) = g(\y) \sum_x h(x\x^\star) = g(\y) \sum_{\substack{x \\ f(x\u)=1}} h(x\x^\star) = g(\y) h(x_\u \x^\star).
}
\]

If \(x_\u\) is the instantiation of \(X\) in \(\z\), then \(x_\u \x^\star = \x\) and \(e_r(\z) = g(\y) h(\x) = e_l(\z)\).
Otherwise, \(g(\y) = 0\) since \(f \in \GG\), which leads to \(e_l(\z) = e_r(\z) = 0\). Hence, \(e_l(\z) = e_r(\z)\) for every
instantiation \(\z\) and we have \(\GG \cdot \HH = \GG \sum_X \HH\). 
\end{proof}

Theorem~\ref{theo:ve2} has a key corollary.
If a functional CPT for variable \(X\) appears in both parts of a product,
we can sum out variable \(X\) from the product by independently summing it out
from each part.
\begin{corollary}\label{coro:ve2}
Consider a functional CPT \(f\) for variable \(X\). If  \(f \in \GG\) and \(f \in \HH\), then
\(\sum_X \GG \cdot \HH = \left(\sum_X \GG\right) \left(\sum_X \HH\right)\). 
\end{corollary}
\begin{proof}
\(\sum_X \GG \cdot \HH = \sum_X \left(\GG \sum_X \HH\right)\) by Theorem~\ref{theo:ve2},
which equals \(\left(\sum_X \HH\right) \left(\sum_X \GG\right)\) by Theorem~\ref{theo:ve1}.
\end{proof}
Theorem~\ref{theo:ve2} and Corollary~\ref{coro:ve2} may appear unusable as they
require multiple occurrences of a functional CPT whereas the factors of a Bayesian network 
contain a single (functional) CPT for each variable. This is where the second result comes in:
{\em duplicating} a functional CPT in a product of factors does not change the product value. 

\begin{thm}\label{theo:ve3}
For functional CPT \(f\), if \(f \in \GG\), then  \(f \cdot \GG=\GG\).
\end{thm}
\begin{proof}
Let \(g(\Z)\) be the product of factors in \(\GG\) and let \(h = f \cdot g\).
Suppose factor \(f\) is the CPT of variable \(X\) and parents \(\U\). 
Consider an instantiation \(\z\) and suppose it includes instantiation \(x\u\). 
If \(f(x\u)=0\), then \(g(\z)=0\) since \(f \in \GG\). Moreover, \(h(\z) = f(x\u) g(\z) = 0\).
If \(f(x\u)=1\), then \(h(\z) = f(x\u) g(\z) = g(\z)\). Hence, \(g(\z) = h(\z)\) for all instantiations \(\z\)
and we have \(\GG= f \cdot \GG\).
\end{proof}
Theorem~\ref{theo:ve3} holds if \(f\) embeds any functional dependency that is implied 
by factors \(\GG\) instead of being a functional CPT in \(\GG\) but we do not pursue the applications
of this generalization in this paper.

To see how Theorems~\ref{theo:ve2} and~\ref{theo:ve3} interplay, 
consider the f-expression \(\sum_X f(XY) g(XZ) h(XW)\). In the standard VE algorithm,
one must multiply all three factors before summing out variable \(X\), leading to a factor over four variables \(XYZW\).
If factor \(f\) is a functional CPT for variable \(X\), we can duplicate it by Theorem~\ref{theo:ve3}:
\(f(XY) g(XZ) h(XW) =  f(XY) g(XZ) f(XY) h(XW)\). Moreover, Corollary~\ref{coro:ve2} gives
\(\sum_X f(XY) g(XZ)f(XY) h(XW) = \sum_X f(XY) g(XZ) \sum_X f(XY) h(XW)\), which
avoids constructing a factor over four variables. We show in Section~\ref{sec:alg}
how these theorems enable efficient inference on models with very large treewidth. 

\section{Variable Elimination with Functional CPTs}
\label{sec:alg}

\def\myfigscale{0.16}
\begin{figure}[tb]
\centering
\includegraphics[scale=\myfigscale]{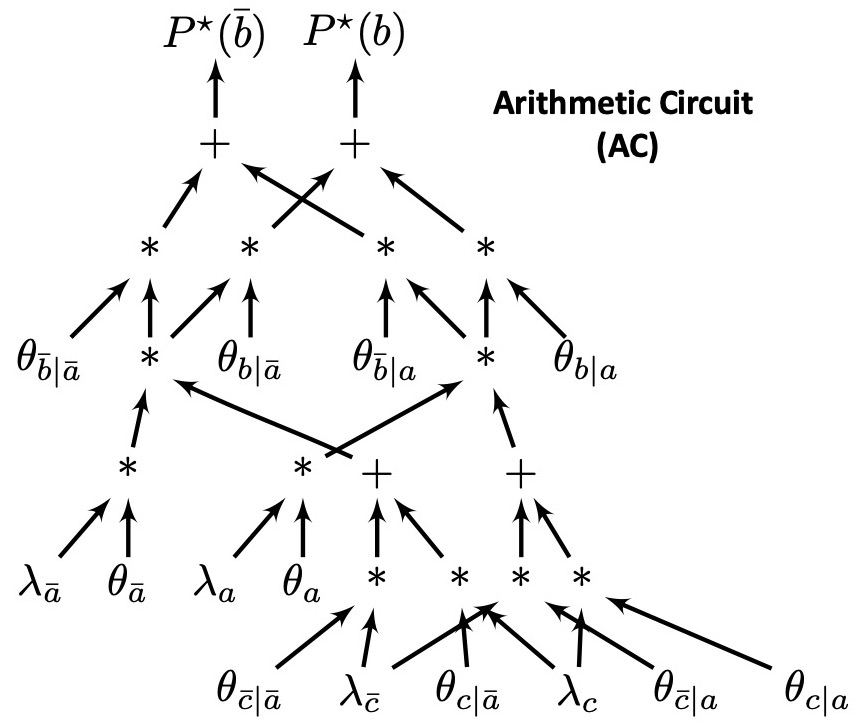}
\caption{An arithmetic circuit (AC) compiled from the Bayesian network \(A \rightarrow B\), \(A \rightarrow C\). 
The AC computes factor \(f(B)\), where \(\NC f\) is the posterior on variable \(B\) 
given evidence on variables \(A\) and \(C\).
\label{fig:ac}}
\end{figure}

We now present our proposed VE algorithm.
We first present a standard VE algorithm based on {\em jointrees}~\cite{JT_ALG}
and then extend it to exploit functional CPTs. 
Our algorithm will not compute probabilities, but will compile {\em symbolic f-expressions} whose
factors contain {\em symbolic parameters.} A symbolic f-expression is compiled once and used thereafter
to answer multiple queries. Moreover, its parameters can be learned
from labeled data using gradient descent. We will show how to map symbolic f-expressions into 
tensor graphs in Section~\ref{sec:tensors} and use these graphs for supervised learning in Section~\ref{sec:studies}.

Once the factors of a symbolic f-expression are unfolded, the result is an Arithmetic Circuits (ACs)~\cite{Darwiche03,ChoiDarwiche17}
as shown in Figure~\ref{fig:ac}.
In fact, the standard VE algorithm we present next is a refinement on the one proposed 
in \cite{Darwiche03} for extracting ACs from jointrees.

The next section introduces jointrees and some key concepts that we need for the standard
and extended VE algorithms.

\subsection{Jointrees}

\def\myfigscale{0.20}
\begin{figure}[tb]
\centering
 \includegraphics[scale=\myfigscale]{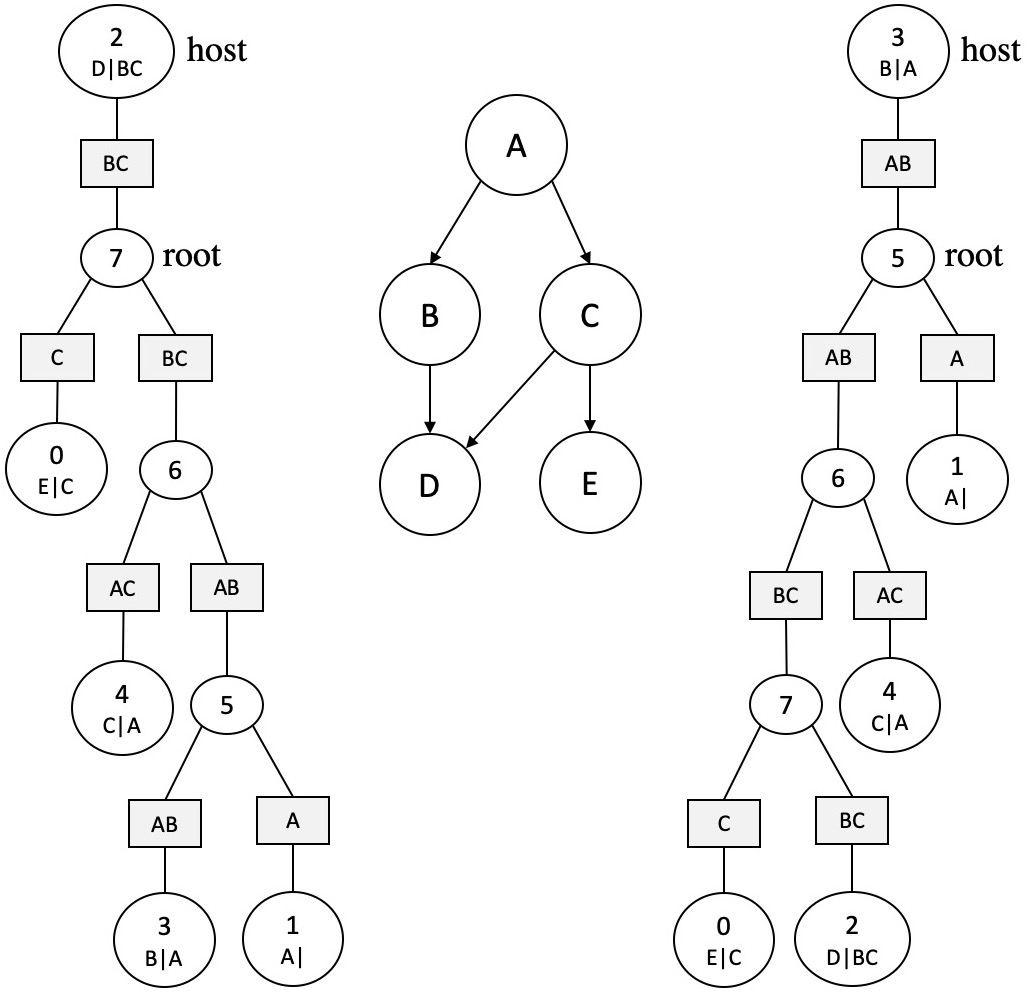}
 \caption{A Bayesian network with a jointree (two views).
 \label{fig:bn-vw}}
\end{figure}

Consider the Bayesian network in the middle of Figure~\ref{fig:bn-vw} and its jointree on the left
of the figure. The jointree is simply a tree with factors attached to some of 
its nodes (the circles in Figure~\ref{fig:bn-vw} are the jointree nodes). 
We use {\em binary jointrees}~\cite{DBLP:conf/uai/Shenoy96}, in which each node has either one or three 
neighbors and where nodes with a single neighbor are called {\em leaves.} 
The two jointrees in Figure~\ref{fig:bn-vw} are identical but arranged differently. 
The one on the left has leaf node~\(2\) at the top and the one on the right has leaf node~\(3\)
at the top.

Our use of jointrees deviates from normal for reasons that become apparent later. First, we use a binary jointree
whose leaves are in one-to-one correspondence with model variables. 
Second, we only attach factors to leaf nodes: The CPT and evidence indicator for each variable \(X\) 
are assigned to the leaf node \(i\) corresponding to variable \(X\). Leaf jointree node \(i\) is 
called the {\em host} of variable \(X\) in this case.\footnote{For similar uses and a method for
constructing such binary jointrees, see~\cite{DarwicheRC} and~\cite[Chapter 8]{Darwiche09}.
{\em Contraction trees} which were adopted later for contracting tensor networks~\cite{ContractionTrees}
correspond to binary jointrees.}

The Bayesian network in Figure~\ref{fig:bn-vw} has five variables. Its jointree also has five leaves,
each of which hosts a network variable. For example, jointree node~\(2\)
at the top-left hosts variable \(D\): the CPT and evidence indicator for
variable \(D\) are assigned to this jointree node.

A key notion underlying jointrees are {\em edge separators} which determine the space complexity 
of inference (the rectangles in Figure~\ref{fig:bn-vw} are
separators). The separator for edge \((i,j)\), denoted \(\sep(i,j)\), are model variables that appear 
in leaf nodes on both sides of the edge. For example, \(\sep(6,7) = \{B,C\}\) as these are the 
model variables that appear in jointree leaves \(\{0,2\}\) and \(\{1,3,4\}\). A related notion
is the {\em cluster} of jointree node \(i\). If \(i\) is leaf, its cluster are the variables appearing at node \(i\).
Otherwise, it is the union of separators for edges \((i,j)\). Every factor constructed by VE
is over the variables of some separator or cluster. The time complexity of VE
is exponential in the size of clusters and linear in the number of nodes in a jointree.

The size of largest cluster \(-1\) is called the {\em jointree width} and cannot be lower than 
the Bayesian network treewidth; see \cite[Chapter 9]{Darwiche09} for a detailed treatment of this subject. 
When the network contains variables with different cardinalities, the size of a cluster is better
measured by the number of instantiations that its variables has. We therefore define the 
{\em binary rank} of a cluster as {\em log2} of its instantiation count. The binary rank 
coincides with the number of variables in a cluster when all variables are binary.

Our technique for exploiting functional dependencies will 
use Theorems~\ref{theo:ve2} and~\ref{theo:ve3} to shrink the size of clusters and separators
significantly below jointree width, allowing us to handle networks with very large treewidth.
The algorithm will basically reduce the maximum binary rank of clusters and separators,
which can exponentially reduce the size of factors constructed by VE during inference.

\subsection{Compiling Symbolic f-expressions using VE}
Suppose we wish to compile an f-expression that computes the posterior on variable \(Q\).
We first identify the leaf jointree node \(h\) that hosts variable \(Q\).
We then arrange the jointree so {\em host} \(h\) is at the top as in Figure~\ref{fig:bn-vw}.
Host \(h\) will then have a single child \(r\) which we call the jointree {\em root.}
The tree rooted at node \(r\) is now a binary tree, with each node \(i\) having two children
\(c_1\) and \(c_2\) and a parent \(p\). On the left of Figure~\ref{fig:bn-vw}, 
root \(\eql(r,7)\) has two children \(\eql(c_1,0)\), \(\eql(c_2,6)\) and parent \(\eql(p,2)\). 
We refer to such a jointree arrangement as a {\em jointree view.}

Jointree views simplify notation. For example, we can now write \(\sep(i)\) to denote
the separator between node \(i\) and its parent \(p\) instead of \(\sep(i,p)\). We will adopt this
simpler notation from now on.

We now compile an f-expression using the following equations:
\begin{eqnarray}
\P(Q) & = & \NC\ \Sum_{Q} \FF_h f(r) \label{eq:jt1} \\
f(i) & = &
\left\{
\begin{array}{ll}
\displaystyle \Sum_{\sep(i)} \FF_i & \mbox{\(i\) is leaf} \\
\displaystyle \Sum_{\sep(i)} f(c_1) f(c_2) & \mbox{\(i\) has children \(c_1,c_2\)}
\end{array}
 \right.  \label{eq:jt2}
\end{eqnarray}
Here, \(\FF_i\) is the product of factors assigned to leaf node \(i\) (CPT and evidence
indicator for the model variable assigned to node \(i\)).

For the jointree view in Figure~\ref{fig:bn-vw} (left), applying these equations to variable 
\(\eql(Q,D)\), host \(\eql(h,2)\) and root \(\eql(r,7)\) yields the f-expression:
\[
\scalemath{0.90}{
\P(D) = \NC \Sum_D  \FF_2 \Sum_{BC}  [\Sum_C \FF_0] [\Sum_{BC} [\Sum_{AC} \FF_4] [\Sum_{AB}[\Sum_{AB} \FF_3] [\Sum_A \FF_1]]].
}
\]
This expression results from applying Equation~\ref{eq:jt1} to the host \(\eql(h,2)\) followed by
applying Equation~\ref{eq:jt2} to each edge in the jointree. Each sum in the expression 
corresponds to a separator and every product constructed by the expression will be over the variables of a cluster. 

Our compiled AC is simply the above f-expression. The value of the expression represents the circuit output. 
The evidence indicators in the expression represent the circuit inputs. Finally, the CPTs of the expression 
contain the circuit parameters (see the AC in Figure~\ref{fig:ac}).

We will now introduce new notation to explain Equations~\ref{eq:jt1} and~\ref{eq:jt2} as we need this understanding in 
the following section; see also~\cite[Chapter 7]{Darwiche09}.
For node~\(i\) in a jointree view, we use \(\facsn(i)\) to denote the set of factors at or below node \(i\).
We also use \(\facsp(i)\) to denote the set of factors above node \(i\).
Consider node~\(6\) on the left of Figure~\ref{fig:bn-vw}. 
Then \(\facsn(6)\) contains the factors assigned to leaf nodes \(\{1,3,4\}\) and \(\facsp(6)\) contains
the factors assigned to leaf nodes \(\{0,2\}\). 

For a jointree view with host \(h\) and root \(r\),
\(\facsp(r)\facsn(r)\) contains all factors in the jointree and \(\facsp(r) = \FF_h\).
Equation~\ref{eq:jt1} computes \(\NC\ \Sum_{Q}\facsp(r)\facsn(r)\), while delegating
the computation of product \(\facsn(r)\) to Equation~\ref{eq:jt2}, which actually computes \(\Sum_{\sep(r)} \facsn(r)\) by
summing out all variables but for ones in \(\sep(r)\). The equation
uses the decomposition \(\facsn(i) = \facsn(c_1)\facsn(c_2)\) to sum out variables more aggressively:
\begin{eqnarray}
f(i) & = & \Sum_{\sep(i)} \facsn(i) = \Sum_{\sep(i)} \facsn(c_1)\facsn(c_2) \label{eq:dec}\\ \nonumber
& = & \Sum_{\sep(i)} \left(\Sum_{\sep(c_1)} \facsn(c_1)\right) \left(\Sum_{\sep(c_2)} \facsn(c_2)\right).
\end{eqnarray}
The rule employed by Equation~\ref{eq:jt2} is simple: sum out from product \(\facsn(i)\) 
all variables except ones appearing in product \(\facsp(i)\) (Theorem~\ref{theo:ve1}). 
The only variables shared between factors \(\facsn(i)\) and \(\facsp(i)\) are the ones in \(\sep(i)\) so 
Equation~\ref{eq:jt2} is exploiting Theorem~\ref{theo:ve1} to the max. 
The earlier that variables are summed out, the smaller the factors we need to multiply and the
smaller the f-expressions that VE compiles.


\subsection{Exploiting Functional Dependencies}
\label{sec:exploit-f}

\def\shrinka{{\sc shrink\_sep}}
\def\fsum{{\sc sum}}

\begin{figure}[tb]
\centering
\begin{tabular}{cc}
\hspace{-3.3mm}
\noindent\begin{minipage}[t]{.4\textwidth}
\begin{footnotesize}
\alglanguage{pseudocode}
\begin{algorithmic}[1]
\Procedure{shrink\_sep}{$r,h$}
\State \(X \gets\) variable assigned to host \(h\)
\If{\(X \in \fvars(r)\)}
\State \(\sep(r)\) -= \(\{X\}\) \label{ln:rs}
\EndIf
\State \Call{sum}{$r$}
\EndProcedure
\vspace{3mm}
\Procedure{sum}{$i$}
\If{leaf node \(i\)}
\State \Return
\EndIf
\State \(c_1, c_2 \gets\) children of node \(i\)
\State \(\X \gets \fvars(c_1)\cap\fvars(c_2)\) \label{ln:s}
\State \(c \gets\) either \(c_1\) or \(c_2\) \label{ln:choice}
\State \(\sep(c)\) -= \(\X\) \label{ln:drop}
\State \(\sep(c_1)\) \&= \(\sep(c_2) \cup \sep(i)\) \label{ln:p1}
\State \(\sep(c_2)\) \&= \(\sep(c_1) \cup \sep(i)\) \label{ln:p2}
\State \Call{sum}{$c_1$}
\State \Call{sum}{$c_2$}
\EndProcedure
\end{algorithmic}
\end{footnotesize}
\end{minipage} & 
\hspace{3mm}
\raisebox{-3cm}{\begin{minipage}{.43\textwidth}
\hspace{-3.2cm}\includegraphics[width=.5\textwidth]{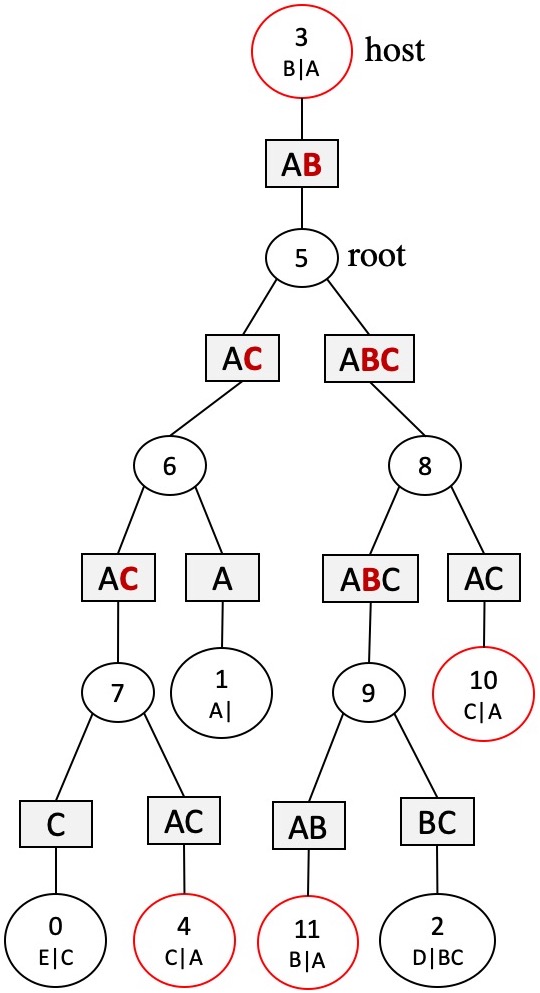}
\end{minipage}}
\end{tabular}
 \caption{{\bf Left:} Algorithm for shrinking separators based on functional CPTs.
 {\bf Right:} An application of the algorithm where dropped variables are colored red.
 Variables \(B\) and \(C\) have functional CPTs.
 \label{fig:alg}}
 
\end{figure}

We now present an algorithm that uses Theorems~\ref{theo:ve2} and~\ref{theo:ve3} 
to sum out variables earlier than is licensed by Theorems~\ref{theo:ve0} 
and~\ref{theo:ve1}. Here, `earlier' means lower in the jointree view which leads
to smaller factors. 

Our algorithm uses the notation \(\fvars(i)\) to denote the set of variables
that have a functional CPT at or below node \(i\) in the jointree view. For example,
in Figure~\ref{fig:alg}, we have \(\fvars(8) = \{B,C\}\), \(\fvars(11) = \{B\}\) and \(\fvars(2) = \{\}\).

The algorithm is depicted in Figure~\ref{fig:alg} and is a direct application of Theorem~\ref{theo:ve2}
with a few subtleties. The algorithm traverses the jointree view top-down, removing variables
from the separators of visited nodes. It is called on root \(r\) and host \(h\) of the view,
\shrinka($r,h$). It first shrinks the separator of root \(r\) which decomposes
the set of factors into \(\facsp(r)\facsn(r)\). The only functional CPT that can be shared between
factors \(\facsp(r)\) and \(\facsn(r)\) is the one for variable \(X\) assigned to host~\(h\).
If variable \(X\) is functional and its CPT is shared, Theorem~\ref{theo:ve2} immediately 
gives \(\facsp(r)\facsn(r) = \facsp(r)\sum_X\facsn(r)\). Variable \(X\) can then be summed 
at root \(r\) by dropping it from \(\sep(r)\) as done on line~\ref{ln:rs}.

The algorithm then recurses on the children of root \(r\). The algorithm processes both children
\(c_1\) and \(c_2\) of a node before it recurses on these children. This is critical as we explain
later. The set \(\X\) computed on line~\ref{ln:s} contains variables that have functional
CPTs in both factors \(\facsn(c_1)\) and factors \(\facsn(c_2)\) (recall Equation~\ref{eq:dec}). 
Theorem~\ref{theo:ve2} allows us to sum out these variables from either \(\facsn(c_1)\)
or \(\facsn(c_2)\) but not both, a choice that is made on line~\ref{ln:choice}. 
A variable that has a functional CPT in both \(\facsn(c_1)\) and \(\facsn(c_2)\) is summed out
from one of them by dropping it from either \(\sep(c_1)\) or \(\sep(c_2)\) on line~\ref{ln:drop}.
In our implementation, we heuristically choose a child based on the size of separators below it. We add
the sizes of these separators (number of instantiations) and choose the child with the largest 
size breaking ties arbitrarily. 

If a variable is summed out at node \(i\) and at its child \(c_2\), we can sum it out earlier at child \(c_1\) by
Theorem~\ref{theo:ve1} (classical VE):
\(\sum_X (\facsn(c_1)\sum_X\facsn(c_2)) = (\sum_X \facsn(c_1))(\sum_{X}\facsn(c_2))\).
A symmetric situation arrises for child \(c_2\). This is handled on lines~\ref{ln:p1}-\ref{ln:p2}. 
Applying Theorem~\ref{theo:ve1} in this context
demands that we process nodes \(c_1\) and \(c_2\) before we process their children. Otherwise,
the reduction of separators \(\sep(c_1)\) and \(\sep(c_2)\) will not propagate downwards early
enough, missing opportunities for applying Theorem~\ref{theo:ve1} further.

Figure~\ref{fig:alg} depicts an example of applying algorithm \shrinka\ to a jointree view
for the Bayesian network in Figure~\ref{fig:bn-vw}. Variables colored red are dropped
by \shrinka. The algorithm starts by processing
root \(r=5\), dropping variable \(B\) from \(\sep(5)\) on line~\ref{ln:rs}. It then processes
children \(c_1=6\) and \(c_2=8\) simultaneously. Since both children contain a functional
CPT for variable \(C\), the variable can be dropped from either \(\sep(6)\) or \(\sep(8)\).
Child \(c_2=8\) is chosen in this case and variable \(C\) is dropped from \(\sep(8)\).
We have \(\sep(6)=\{A,C\}\) and \(\sep(8)=\{A,B\}\) at this point. Lines~\ref{ln:p1}-\ref{ln:p2}
shrink these separators further to \(\sep(6)=\{A\}\) and \(\sep(8)=\{A\}\).

Our proposed technique for shrinking separators will have an effect only when functional CPTs 
have multiple occurrences in a jointree (otherwise, set \(\X\) on line~\ref{ln:s} is always empty). 
While this deviates from the standard use of jointrees, replicating functional CPTs is licensed 
by Theorem~\ref{theo:ve3}. 
The (heuristic) approach we adopted for replicating functional CPTs in a jointree is 
based on replicating them in the Bayesian network. Suppose variable \(X\) has a functional
CPT and children \(C_1, \ldots, C_n\) in the network, where \(n > 1\). We replace variable \(X\) 
with replicas \(X_1, \ldots, X_n\). Each replica \(X_i\) has a single child \(C_i\)
and the same parents as \(X\). We then construct a jointree for the resulting network
and finally replace each replica \(X_i\) by \(X\) in the jointree.
This creates \(n\) replicas of the functional CPT in the jointree.
Replicating functional CPTs leads to jointrees with more nodes, but smaller separators
and clusters as we shall see in Section~\ref{sec:studies}.

\section{Mapping ACs into Tensor Graphs}
\label{sec:tensors}

\def\pytac{{\sc PyTAC}}

We discuss next how we map ACs (symbolic f-expressions) into tensors graphs for
efficient inference and learning. Our implementation is part of the \pytac\ system
under development by the author. \pytac\ is built on top of TensorFlow 
and will be open sourced.

A {\em tensor} is a data structure for a multi-dimensional array. 
The {\em shape} of a tensor defines the array dimensions.
A tensor with shape \((2,2,3)\) has \(2 \times 2 \times 3\) {\em elements} or {\em entries.} 
The dimensions of a tensor are numbered and called {\em axes.}
The number of axes is the tensor {\em rank.}
Tensor computations can be organized into a {\em tensor graph:} a data
flow graph with nodes representing tensor operations. 
Tensors form the basis of many machine learning tools today.

A factor over variables \(X_1, \ldots, X_n\) can be represented by a tensor with rank \(n\) and
shape \((d_1, \ldots, d_n)\), where \(d_i\) is the cardinality of variable \(X_i\) (i.e., its number
of values). Factor operations can then be implemented using tensor operations, leading
to a few advantages. 
First, tensor operations are heavily optimized to take advantage of special instruction sets
and architectures  (on CPUs and GPUs) so they can be orders of magnitude faster than standard 
implementations of factor operations (even on laptops). 
Next, the elements of a tensor can be variables, allowing one to represent symbolic 
f-expressions, which is essential for mapping ACs into tensor graphs that can be trained. 
Finally, tools such as TensorFlow and PyTorch provide support for computing the partial derivates 
of a tensor graph with respect to tensor elements, and come with effective gradient 
descent algorithms for optimizing tensor graphs (and hence ACs). 
This is very useful for training ACs from labeled data as we do in Section~\ref{sec:studies}.

To map ACs (symbolic f-expressions) into tensor graphs, we need to implement
factor multiplication, summation and normalization.
Mapping factor summation and normalization into tensor operations is 
straightforward: summation has a corresponding tensor operation ({\sc tf.reduce\_sum}) 
and normalization can be implemented using tensor summation and division.
Factor multiplication does not have a corresponding tensor operation and leads to some 
complications.\footnote{Tensor multiplication is pointwise while 
factors are normally over different sets of variables. Hence,
multiplying the tensors corresponding to factors \(f(ABC)\) and \(g(BDE)\) does not yield 
the expected result. The simplest option is to use {\sc tf.einsum}, which
can perform factor multiplication if we pass it the string \mbox{``abc, bde --\(>\) abcde''}
(\url{https://www.tensorflow.org/api_docs/python/tf/einsum}).
We found this too inefficient though for extensive use as it performs too many tensor 
transpositions. One can also use the technique of {\em broadcasting}
by adding {\em trivial} dimensions to align tensors (\url{https://www.tensorflow.org/xla/broadcasting}),
but broadcasting has limited support in TensorFlow requiring tensors with small enough ranks.}

We bypassed these complications in the process of achieving something more ambitious.
Consider Equation~\ref{eq:jt2} which contains almost all multiplications performed by VE.
Factors \(f_1(c_1)\), \(f_2(c_2)\) and the result \(f(i)\) are over separators 
\(\sep(c_1)\), \(\sep(c_2)\) and \(\sep(i)\). This equation multiplies factors \(f_1\) and \(f_2\)
to yield a factor over variables \(\sep(c_1) \cup \sep(c_2)\) and then shrinks it by
summation into a factor over variables \(\sep(i)\). We wanted to avoid constructing
the larger factor before shrinking it. That is, we wanted to multiply-then-sum in one shot
as this can reduce the size of our tensor graphs significantly.\footnote{See a discussion
of this space issue in~\cite[Chapter 7]{Darwiche09}.} A key observation allows
this using standard tensor operations. 

\begin{figure*}
\centering
\includegraphics[scale=.105]{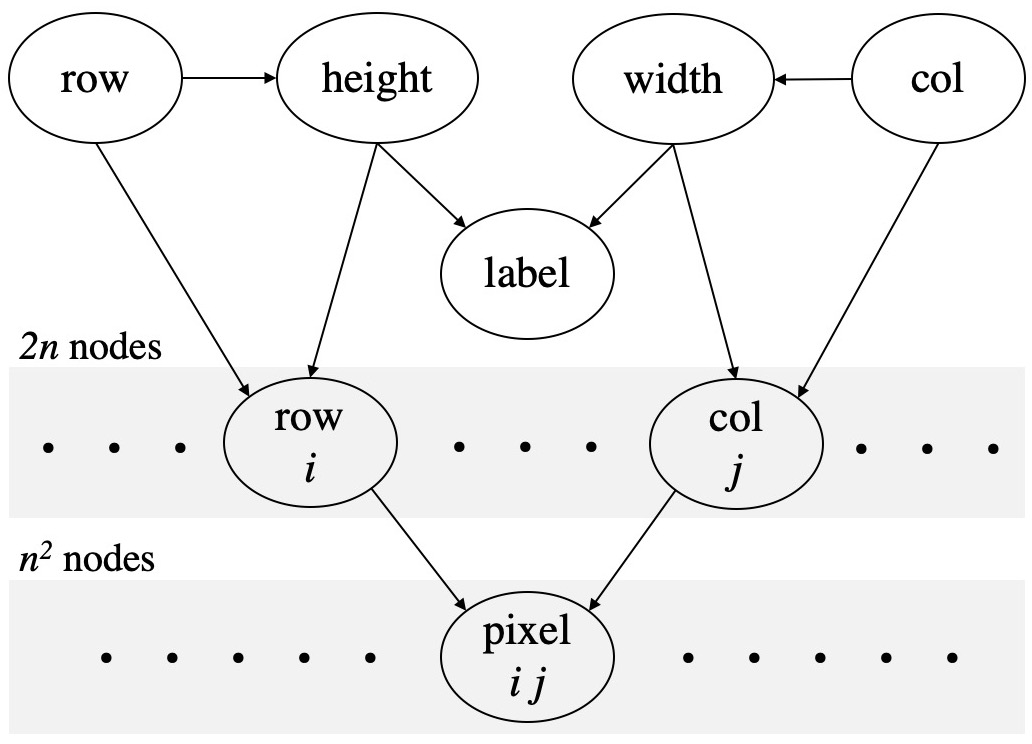}
\hspace{5mm}
\includegraphics[scale=.175]{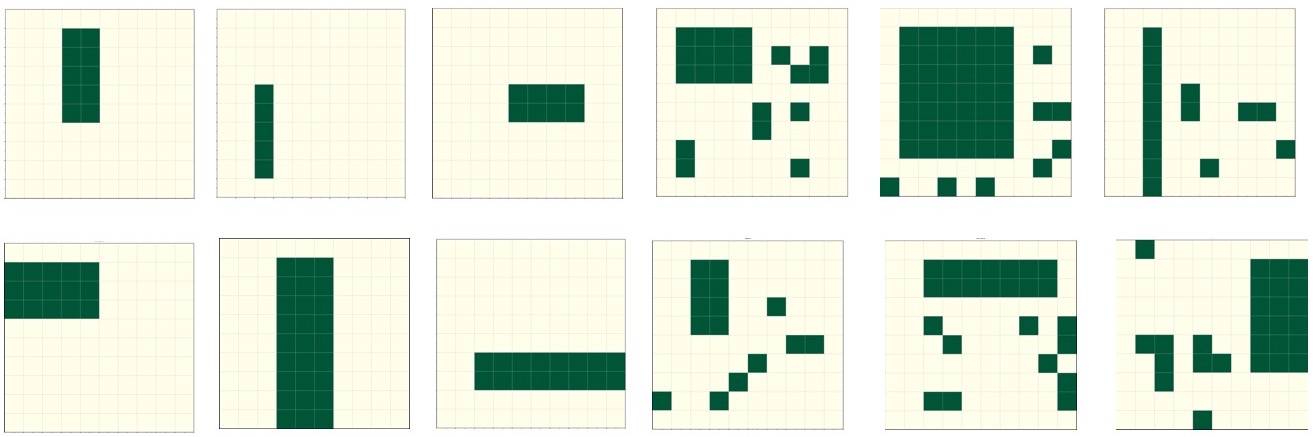}
\caption{{\bf Left:} A generative model for rectangles. {\bf Right:} Examples of clean and noisy rectangle images.
\label{fig:rect-model} \label{fig:rect-images}}
\centering
\includegraphics[scale=.1]{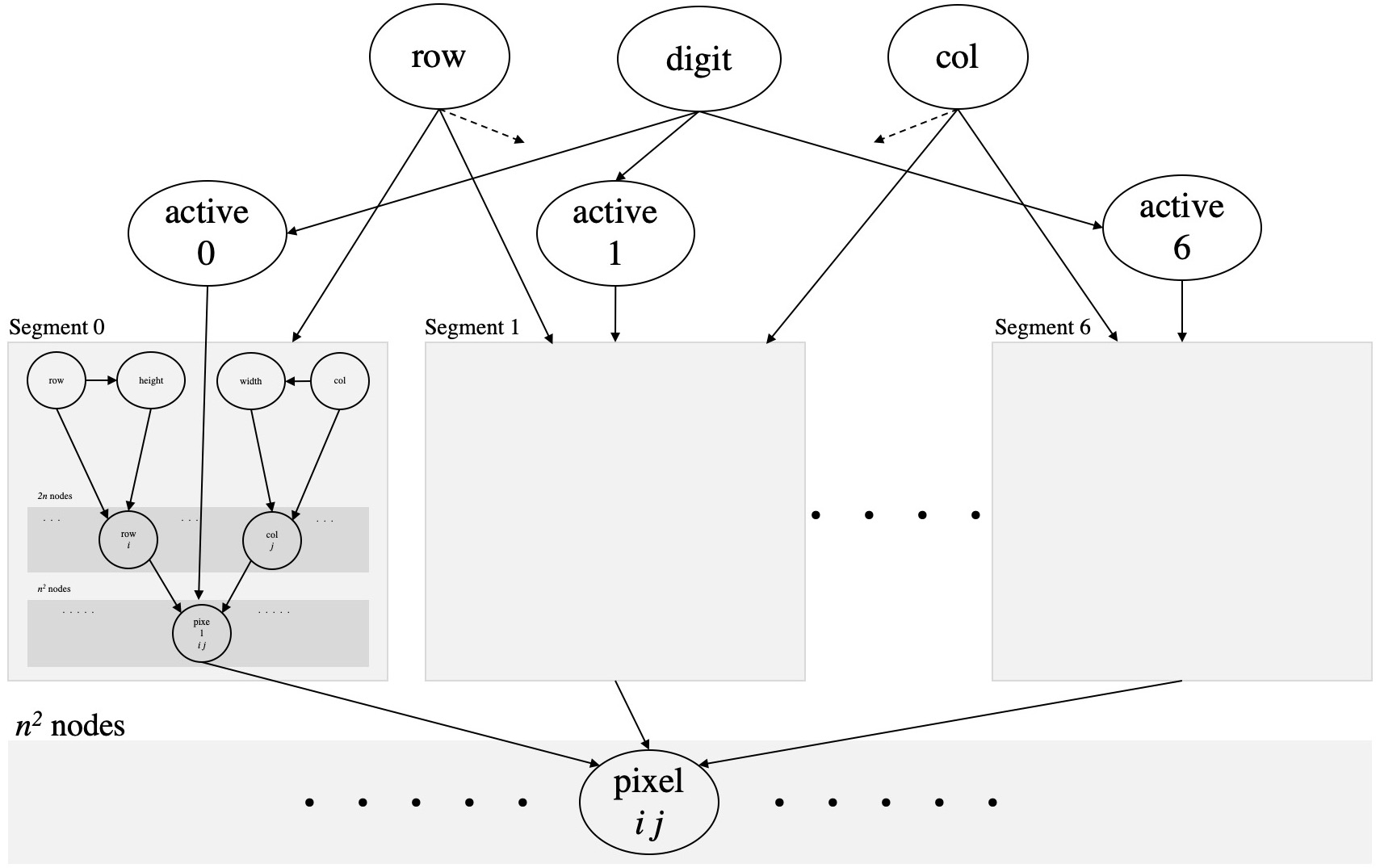}
\hspace{5mm}
\includegraphics[scale=.116]{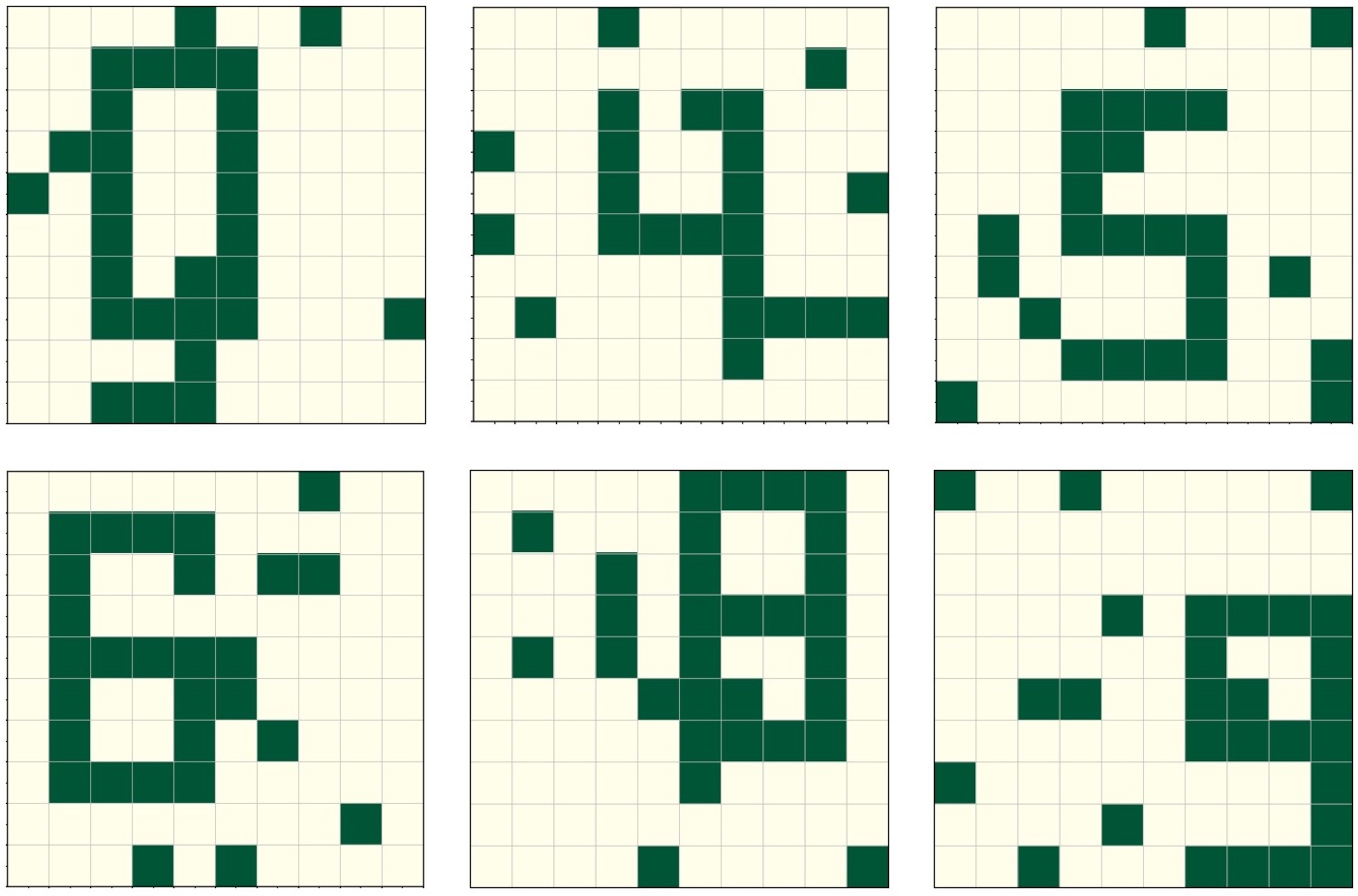}
\hspace{5mm}
\includegraphics[scale=.195]{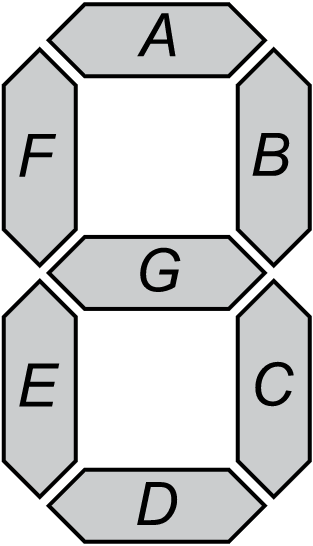}
\caption{{\bf Left:} A generative model for seven-segment digits. {\bf Middle:} Examples of noisy digit images. {\bf Right:} Seven-segment digit.
\label{fig:digits-model} \label{fig:digit-images}}
\end{figure*}

\def\S{{\bf S}}

The previous separators are all connected to jointree node \(i\) so they satisfy the following property~\cite[Chapter 9]{Darwiche09}:
If a variable appears in one separator, it also appears in at least one other separator.
Variables \(\sep(c_1)\cup\sep(c_2)\cup\sep(i)\) can then be partitioned as follows:\footnote{In a jointree,
every separator that is connected to a node is a subset of the union of other separators connected
to that node. Hence, \(\sep(i) \subseteq \sep(c_1)\cup\sep(c_2)\).}
\begin{itemize}
\item[] \(\C\): variables in \(f_1, f_2\) and \(f\), \(\sep(c_1)\cap\sep(c_2)\cap\sep(i)\)
\item[] \(\X\): variables in \(f_1, f\) but not \(f_2\), \((\sep(c_1)\cap\sep(i))\setminus\sep(c_2)\)
\item[] \(\Y\): variables in \(f_2, f\) but not \(f_1\), \((\sep(c_2)\cap\sep(i))\setminus\sep(c_1)\)
\item[] \(\S\): variables in \(f_1, f_2\) but not \(f\), \((\sep(c_1)\cap\sep(c_2))\setminus\sep(i)\)
\end{itemize}
where variables \(\S\) are the ones summed out by Equation~\ref{eq:jt2}.
The variables in each factor can now be structured as follows: \(f_1(\C,\X,\S)\), \(f_2(\C,\Y,\S)\)
and \(f(\C,\X,\Y)\). We actually group each set of variables \(\C, \X, \Y\) and \(\S\) into a single
compound variable so that factors \(f_1, f_2\) and \(f\) can each be represented by a rank-\(3\)
tensor. We then use the tensor operation for matrix multiplication {\sc tf.matmul}
to compute \(f = \sum_\S f_1 f_2\) in one shot, without having to construct a tensor 
for the product \(f_1 f_2\). Matrix multiplication is perhaps one of the most optimized
tensor operations on both CPUs and GPUs. 

Preparing tensors  \(f_1(\C,\X,\S)\) and \(f_2(\C,\Y,\S)\) for matrix 
multiplication requires two operations: {\sc tf.reshape} which aggregate variables into compound dimensions
and {\sc tf.transpose} which order the resulting
dimensions so {\sc tf.matmul} can map \(f_1\) and \(f_2\) into \(f(\C,\X,\Y)\). The common dimension
\(\C\) must appear first in \(f_1\) and \(f_2\). Moreover, the last two dimensions must
be ordered as \((\X,\S)\) and \((\S,\Y)\) but {\sc tf.matmul} can transpose the
last two dimensions of an input tensor on the fly if needed.
Using matrix multiplication in this fashion had a significant impact on reducing the size of tensor graphs and the efficiency 
of evaluating them, despite the added expense of using {\sc tf.transpose}
and {\sc tf.reshape} operations (the latter operation does not use space and  is very efficient).

\pytac\ represents ACs using an {\em abstract} tensor graph called an {\em ops graph,} which 
can be mapped into a particular tensor implementation depending on the used machine learning tool.
\pytac\ also has a dimension management utility, which associates each tensor with its 
structured dimensions while ensuring that all tensors are structured appropriately so operations 
can be applied to them efficiently. We currently map an ops graph into a {\sc tf.graph} object, using the
{\sc tf.function} utility introduced recently in TensorFlow 2.0.0.
\pytac\ also supports the recently introduced {\em Testing Arithmetic Circuits (TACs),} which
augment ACs with {\em testing units} that turns them into universal function approximators 
like neural networks~\cite{TBN19,DBLP:conf/pgm/ChoiD18,DBLP:conf/icml/ShenHCD19}.

\section{Case Studies}
\label{sec:studies}

We next evaluate the proposed VE algorithm on two classes of models that have abundant functional dependencies. 
We also evaluate the algorithm on randomly generated Bayesian networks while varying the amount of functional dependencies.
The binary jointrees constructed for these models are very large and prohibit inference using standard VE.
We constructed these binary jointrees from variable elimination orders 
using the method proposed in~\cite{DarwicheRC}; see also~\cite[Chapter 9]{Darwiche09}.
The elimination orders were obtained by the {\em minfill heuristic;} see, e.g.,~\cite{jointreeCONST}.\footnote{The minfill heuristic 
and similar ones aim for jointrees that minimize the size of largest cluster (i.e., treewidth). It was observed 
recently that minimizing the size of largest separator (called {\em max rank}) is more desirable when
using tensors since the memory requirements of Equation~\ref{eq:jt2} can
depend only on the size of separators not clusters (see~\cite{TensorOrder} for recent methods that optimize
max rank). This observation holds even
when using classical implementations of the jointree algorithm and was exploited earlier to reduce the memory requirements
of jointree inference; see, e.g.,~\cite{SSArch,RinaTimeSpace}.}

\subsection{Rectangle Model}
\label{sec:rectangle}

\def\Row{${\tt row}$}
\def\Col{${\tt col}$}
\def\iRow{${\tt row}_i$}
\def\jCol{${\tt col}_j$}
\def\Width{${\tt width}$}
\def\Height{${\tt height}$}
\def\Label{${\tt label}$}
\def\Pixel{${\tt pixel}_{ij}$}

We first consider a generative model for rectangles shown in Figure~\ref{fig:rect-model}. In an image
of size \(n \times n\), a rectangle is defined by its upper left corner (\Row, \Col), \Height\ and \Width. 
Each of these variables has \(n\) values.
The rectangle also has a binary \Label\ variable, which is either tall or wide. 
Each row has a binary variable \iRow\ indicating whether the rectangle will render in that row (\(n\) variables total).
Each column has a similar variable \jCol.
We also have \(n^2\) binary variables which correspond to image pixels (\Pixel) indicating
whether the pixel is on or off. This model can be used to predict rectangle attributes from noisy images 
such as those shown in Figure~\ref{fig:rect-images}. We use the model to predict whether a rectangle 
is tall or wide by compiling an AC with variable \Label\ as output and variables \Pixel\ as input. 
The AC computes a distribution on \Label\ given a noisy image as evidence and can be trained 
from labeled data using cross entropy as the loss function.\footnote{Arthur Choi suggested 
the use of rectangle models and Haiying Huang proposed this particular version of the model.}

Our focus is on the variables \iRow\ and \jCol\ which are determined by \Row/\Height\ 
and \Col/\Width, respectively (for example, \iRow\ is on iff \Row\ \(\leq i  <\) \Row + \Height). 
In particular, we will investigate the impact of these functional relationships 
on the efficiency of our VE compilation algorithm and their impact on learning AC parameters from labeled data. 
Our experiments were run on a MacBook Pro, 2.2 GHz Intel Core i7, with 32 GB RAM.

Table~\ref{tab:rect-size} depicts statistics on ACs that we compiled using our proposed VE algorithm.
For each image size, we compiled an AC for predicting the rectangle \Label\ while exploiting 
functional CPTs to remove variables from separators during the compilation process. 
As shown in the table, exploiting functional CPTs has a dramatic impact on the complexity
of VE. This is indicated by the size of largest jointree cluster (binary rank) in a classical jointree vs one whose separators
and clusters where shrunk due to functional dependencies.\footnote{We applied standard node and value 
pruning to the Bayesian network before computing a jointree and shrinking it. This has more effect
on the digits model in Section~\ref{sec:digits}. For example, it can infer that some pixels will never be
turned on as they will never be occupied by any digit.}
Recall that a factor over a cluster will have a size exponential in the cluster binary rank (the
same for factors over separators).
The table also shows the size of compiled ACs, which is the sum of tensor sizes in
the corresponding tensor graph (the tensor size is the number of elements/entries it has).
For a baseline, the AC obtained by standard VE (without exploiting functional CPTs) for an image
of size \(20 \times 20\) is $18,032,742,365$, which is about \(80\) times larger than
the size of AC reported in Table~\ref{tab:rect-size}. What is particularly impressive is the time
it takes to evaluate these ACs (compute their output from input). On average it takes about \(7\)
milliseconds to evaluate an AC of size ten million for these models, which shows the promise
tensor-based implementations (these experiments were run on a laptop).

We next investigate the impact of integrating background knowledge when learning AC parameters. 
For training, we generated labeled data for all clean images of rectangles and added \(n\) noisy images 
for each (with the same label). Noise is generated by randomly flipping
\(min(n,a-1,b/2)\) background pixels, where \(a\) is the number of rectangle 
pixels and \(b\) is the number of background pixels. We used the same process for testing data,
except that we increased the number of noisy pixels to \(min(2*n,a-1,b/2)\) and doubled the number 
of noisy images. We trained the AC using cross entropy as the loss function to minimize
the classification accuracy.\footnote{Some of the CPTs contain zero parameters but
are not functional, such as the ones for \Width\ and \Height. We fixed these zeros in the AC
when learning with background knowledge. We also tied the parameters of the \Pixel\ variables
therefore learning one CPT for all of them.}

Table~\ref{tab:rect-predict} shows the accuracy of classifying rectangles (tall vs wide) on \(10 \times 10\) 
images using ACs with and without background knowledge. 
ACs compiled from models with background knowledge have fewer parameters and therefore need less
data to train.
The training and testing examples were selected randomly from the datasets described above with
\(1000\) examples always used for testing, regardless of the training data size. 
Each classification accuracy is the average over twenty five runs.  The table clearly shows that integrating 
background knowledge into the compiled AC yields higher classification accuracies given 
a fixed number of training examples.

\begin{table}
\caption{Size and compile/evaluation time for ACs that compute the posterior on rectangle label. 
Reported times are in seconds. Evaluation time is the average of evaluating an AC over a batch of examples.}\smallskip
\centering
\resizebox{1.0\columnwidth}{!}{
\smallskip\begin{tabular}{|c|c|c|c|c|c|c|c|} \hline
Image & Functional &Network & \multicolumn{2}{c|}{Max Cluster Size} & AC & {\bf Eval}  & Compile  \\ 
\cline{4-5}
  Size   & CPTs & Nodes & rank & binary rank & Size & {\bf Time}  & Time \\ \hline
\multirow{2}{*}{\(8 \times 8\)}     
  & \xmark  & $85$ & $11$ & $15.0$ & \multirow{2}{*}{$926,778$} & \multirow{2}{*}{$.001$} & \multirow{2}{*}{$4.9$}  \\
  & \cmark & $197$ & $5$ & $ {\bf 13.0}$ & & & \\ \hline
\multirow{2}{*}{\(10 \times 10\) }  
  & \xmark  & $125$ & $13$ & $ 17.6$ & \multirow{2}{*}{$3,518,848$} & \multirow{2}{*}{$.003$} & \multirow{2}{*}{$2.9$} \\
  & \cmark & $305$ & $5$ & $ {\bf 14.3}$ & & & \\ \hline
\multirow{2}{*}{\(12 \times 12\)}  
  & \xmark  & $173$ & $15$ & $ 20.2$ & \multirow{2}{*}{$10,485,538$} & \multirow{2}{*}{$.007$} & \multirow{2}{*}{$4.1$}\\
  & \cmark & $437$ & $5$ & $ {\bf 15.3}$ & & & \\ \hline
\multirow{2}{*}{\(14 \times 14\)}  
  & \xmark  & $229$ & $17$ & $ 22.6$ & \multirow{2}{*}{$26,412,192$} & \multirow{2}{*}{$.018$} & \multirow{2}{*}{$5.7$}\\
  & \cmark & $593$ & $5$ & $ {\bf 16.2}$ & & & \\ \hline
\multirow{2}{*}{\(16 \times 16\) } 
  & \xmark  & $293$ & $19$ & $ 25.0$ & \multirow{2}{*}{$58,814,458$} & \multirow{2}{*}{$.034$} & \multirow{2}{*}{$7.4$}\\
  & \cmark & $773$ & $5$ & $ {\bf 17.0}$ & & & \\ \hline
  \multirow{2}{*}{\(20 \times 20\) } 
  & \xmark  & $445$ & $23$ & $ 29.6$ &  \multirow{2}{*}{$224,211,138$} & \multirow{2}{*}{$.140$} & \multirow{2}{*}{$14.1$}\\
  & \cmark & $1205$ & $5$ & $ {\bf 18.3}$ & & & \\ \hline
\end{tabular}
}
\label{tab:rect-size}

\caption{Classification accuracy on \(10\!\times\!10\) noisy rectangle images. Testing data included \(1000\) examples in each case.
}\smallskip
\centering
\resizebox{1.0\columnwidth}{!}{
\smallskip\begin{tabular}{|c|c||c|c|c|c|c|c||c|} \hline
Functional & \multirow{2}{*}{Accuracy} & \multicolumn{6}{c||}{Number of Training Examples} & Param \\
\cline{3-8} 
 CPTs &  & $25$ & $50$ & $100$ & $250$ & $500$ & $1000$ & Count \\ \hline
\multirow{2}{*}{fixed in AC}  
                  &  mean  &  $82.64$  &   $89.16$   &  $96.08$  & $97.92$  &  $99.51$   & $98.39$  & \multirow{2}{*}{$136$}  \\ 
                  &  stdev   &  $15.06$  &   $11.98$  &  $8.34$  &  $5.56$  &  $0.62$  &  $7.00$  & $$ \\ \hline
\multirow{2}{*}{trainable}
                  &   mean  &  $53.29$   &  $56.92$   & $62.20$  &  $74.62$  &  $ 84.94$  &  $88.69$ & \multirow{2}{*}{$4,428$} \\ 
                  &  stdev  &   $1.89$    &  $5.31$  &  $ 6.95$  &  $5.29$  &  $3.14$  &  $2.79$  & \\ \hline
\end{tabular}
}
\label{tab:rect-predict}
\end{table}

\begin{table}
\caption{Size and compile/evaluation time for ACs that compute a posterior over digits. Reported times are in seconds.
Evaluation time is the average of evaluating an AC over a batch of examples.}\smallskip
\centering
\resizebox{1.0\columnwidth}{!}{
\smallskip\begin{tabular}{|c|c|c|c|c|c|c|c|} \hline
Image & Functional &Network & \multicolumn{2}{c|}{Max Cluster Size} & AC & {\bf Eval}  & Compile  \\ 
\cline{4-5}
  Size   & CPTs & Nodes & rank & binary rank & Size & {\bf Time}  & Time \\ \hline
\multirow{2}{*}{\(8 \times 8\)}      
  & \xmark  & $638$ & $32$ & $33.3$ & \multirow{2}{*}{$264,357$} & \multirow{2}{*}{$.008$} & \multirow{2}{*}{$9.3$}  \\
  & \cmark & $1155$ & $9$ & $ {\bf 12.6}$ & & & \\ \hline
\multirow{2}{*}{\(10 \times 10\) }  
  & \xmark  & $954$ & $59$ & $ 60.8$ & \multirow{2}{*}{$2,241,205$} & \multirow{2}{*}{$.008$} & \multirow{2}{*}{$13.6$} \\
  & \cmark & $2173$ & $9$ & $ {\bf 14.1}$ & & & \\ \hline
\multirow{2}{*}{\(12 \times 12\)}  
  & \xmark  & $1334$ & $81$ & $ 83.8$ & \multirow{2}{*}{$11,625,558$} & \multirow{2}{*}{$.014$} & \multirow{2}{*}{$23.2$}\\
  & \cmark & $3469$ & $10$ & $ {\bf 16.7}$ & & & \\ \hline
\multirow{2}{*}{\(14 \times 14\)}  
  & \xmark  & $1778$ & $116$ & $ 121.0$ & \multirow{2}{*}{$32,057,227$} & \multirow{2}{*}{$.030$} & \multirow{2}{*}{$36.8$}\\
  & \cmark & $5007$ & $11$ & $ {\bf 18.4}$ & & & \\ \hline
\multirow{2}{*}{\(16 \times 16\) } 
  & \xmark  & $2286$ & $134$ & $ 140.0$ & \multirow{2}{*}{$95,094,167$} & \multirow{2}{*}{$.076$} & \multirow{2}{*}{$50.4$}\\
  & \cmark & $6825$ & $11$ & $ {\bf 19.3}$ & & & \\ \hline
\end{tabular}
}
\label{tab:digits-size}

\caption{Classification accuracy on \(10 \times 10\) noisy digit images. Testing data included \(1000\) examples in each case.}\smallskip
\centering
\resizebox{1.0\columnwidth}{!}{
\smallskip\begin{tabular}{|c|c||c|c|c|c|c|c||c|} \hline
Functional & \multirow{2}{*}{Accuracy} & \multicolumn{6}{c||}{Number of Training Examples} & Param \\
\cline{3-8} 
 CPTs &  & $25$ & $50$ & $100$ & $250$ & $500$ & $1000$ & Count \\ \hline
\multirow{2}{*}{fixed in AC}  
                  &  mean  &  $83.51$  &   $89.17$   &  $ 94.94$  & $97.68$  &  $98.49$   & $98.44$  & \multirow{2}{*}{$275$}  \\ 
                  &  stdev   &  $8.70$  &   $6.02$       &  $4.57$    &  $1.45$    &  $0.91$  &  $ 0.27$  & $$ \\ \hline
\multirow{2}{*}{trainable}
                  &   mean  &  $9.82$   &  $12.26$   & $13.28$  &  $22.36$  &  $29.51$  &  $35.67$ & \multirow{2}{*}{$22,797$} \\ 
                  &  stdev  &   $0.77$   &  $2.25$  &  $3.32$  &  $3.45$  &  $2.40$  &  $1.57$  & \\ \hline
\end{tabular}
}
\label{tab:digits-predict}
\end{table}

\subsection{Digits Model}
\label{sec:digits}

We next consider a generative model for seven-segment digits 
shown in Figure~\ref{fig:digits-model} (\url{https://en.wikipedia.org/wiki/Seven-segment_display}).
The main goal of this model is to recognize digits in noisy images such as those shown in Figure~\ref{fig:digit-images}.
The model has four vertical and three horizontal segments.
A digit is generated by activating some of the segments. For example, digit \(8\) is generated by activating all segments
and digit \(1\)  by activating two vertical segments. Segments are represented by rectangles as in the previous section,
so this model integrates seven rectangle models.
A digit has a location specified by the row and column of its upper-left corner (height is seven pixels and width is four pixels). Moreover,
each segment has an activation node which is turned on or off depending on the digit. When this activation node is off, 
segment pixels are also turned off. An image of size \(n \times n\) has \(n^2\) pixels whose
values are determined by the pixels generated by  segments. 

This is a much more complex and larger model than the rectangle model and also has an abundance of functional dependencies. 
It is also much more challenging computationally. This can be seen by examining Tables~\ref{tab:digits-size}, which
reports the size of largest clusters in the jointrees for this model. For example, the model for \(16 \times 16\) images
has a cluster with a binary rank of \(140\). This means that standard VE would have to construct a factor of size 
\(2^{140}\) which is impossible. Our proposed technique for exploiting functional dependencies makes this possible
though as it reduces the binary rank of largest cluster down to \(19.3\). And even though the corresponding AC
has size of about one hundred million, it can be evaluated in about \(76\) milliseconds. The AC compilation times
are also relatively modest.

We trained the compiled ACs as we did in the previous section.
We generated all clean images and added noise as follows.
For each clean image we added \(100\) noisy images for training and \(200\) for testing by
randomly flipping \(n\) background pixels where \(n\) is the image size.

Table~\ref{tab:digits-predict} parallels the one for the rectangle model. We trained two ACs, 
one that integrates background knowledge and one that does not. 
The former AC has fewer parameters and therefore requires less data to train.
While this is expected, it is still interesting to see how little data one needs to get reasonable accuracies.
In general, Tables~\ref{tab:digits-size} and~\ref{tab:digits-predict} reveal the same patterns of the rectangle model:
exploiting functional dependencies leads to a dramatic reduction in the AC size and
integrating background knowledge into the compiled AC significantly improves learnability. 

\subsection{Random Bayesian Networks}
\label{sec:random}

\begin{table}
\caption{Reduction in maximum cluster size due to exploiting functional dependencies.
The number of values a node has was chosen randomly from \((2,3)\).
We averaged over \(10\) random networks for each combination of network node count, maximal parent
count and the percentage of nodes having functional CPTs. The parents of a node
and their count where chosen randomly.
Functional nodes where chosen randomly from non-root nodes.
The {\em binary rank} of a cluster is {\em log2} of the number of its instantiations.}\smallskip
\centering
\resizebox{1,0\columnwidth}{!}{
\smallskip\begin{tabular}{|c|c|c||c|c||c|c||c|c|} \hline
Network & Maximal & Percentage & \multicolumn{6}{c|}{\bf Binary Rank of Largest Cluster}  \\
\cline{4-9}
Node & Parent & Functional & \multicolumn{2}{c||}{Original Jointree} & \multicolumn{2}{c||}{Shrunk Jointree} & \multicolumn{2}{c|}{Reduction}  \\
\cline{4-9} 
 Count &  Count & Nodes~\% & mean & stdev & mean & stdev & {\bf mean}  & stdev \\ \hline \hline
 \multirow{4}{*}{$75$}
           &    \multirow{4}{*}{$4$}
                         &   $25$   &   $22.4$  &   $2.8$   &   $19.4$   &   $3.1$ &   $\bf 3.0$  &  $1.7$  \\ \cline{3-9}
           &            &   $50$   &   $22.5$  &   $2.2$   &   $16.9$   &   $1.8$  &  $\bf 5.6$  &  $2.5$   \\ \cline{3-9}
           &            &   $67$   &   $22.9$  &   $3.9$   &   $13.1$   &   $2.3$  &  $\bf 9.8$  &  $3.4$   \\ \cline{3-9}
           &            &   $80$   &   $21.9$  &   $2.7$   &   $11.1$   &   $1.9$  &  $\bf 10.8$  &  $3.2$   \\ \hline \hline
 \multirow{4}{*}{$100$}
           &     \multirow{4}{*}{$5$}
                         &   $25$   &   $38.7$  &   $4.5$   &   $33.1$   &   $4.6$  &  $\bf 5.7$  &  $2.0$   \\ \cline{3-9}
           &            &   $50$   &   $38.1$  &   $2.9$   &   $23.7$   &   $3.3$  &  $\bf 14.4$  &  $4.3$   \\ \cline{3-9}
           &            &   $67$   &   $38.0$  &   $3.2$   &   $18.9$   &   $3.1$  &  $\bf 19.1$  &  $3.9$   \\ \cline{3-9}
           &            &   $80$   &   $36.8$  &   $3.0$   &   $13.5$   &   $2.5$  &  $\bf 23.3$  &  $3.1$   \\ \hline \hline
 \multirow{4}{*}{$150$}
           &    \multirow{4}{*}{$6$}
                         &   $25$   &   $64.3$  &   $5.4$   &   $54.2$   &   $4.4$  &  $\bf 10.1$  &  $4.2$   \\ \cline{3-9}
           &            &   $50$   &   $64.9$  &   $3.2$   &   $41.9$   &   $5.6$  &  $\bf 23.0$  &  $5.1$   \\ \cline{3-9}
           &            &   $67$   &   $64.3$  &   $6.0$   &   $28.2$   &   $4.2$  &  $\bf 36.0$  &  $4.7$   \\ \cline{3-9}
           &            &   $80$   &   $66.4$  &   $4.8$   &   $21.3$   &   $4.6$  &  $\bf 45.1$  &  $2.1$   \\ \hline 
\end{tabular}
}
\label{tab:reduce-tw}
\end{table}

\begin{table}
\caption{Comparing evaluation time of three  AC representations:
Tensor graph ({\bf TenG}), scalar graph ({\bf ScaG}) and 
scalar-batch graph ({\bf ScaBaG}). 
We averaged over \(10\)
random Bayesian networks for each combination of batch size and limit on circuit size. 
AC size limit is in millions of nodes.
The binary rank of a tensor is {\em log2} of the number of its entries. 
Maximum binary rank is for the largest tensor in the tensor graph. 
Normalized time (tensor graph) is evaluation time per one million AC nodes (a node is a tensor entry).
Each cell below contains the {\bf mean} (top) and {\bf stdev} (bottom).
Times are in milliseconds.
}\smallskip
\centering
\resizebox{1,0\columnwidth}{!}{
\smallskip\begin{tabular}{|c|c||c|c||c||c|c|} \hline
Batch & \multicolumn{3}{c||}{Tensor Graph (TenG)} & {\bf Milliseconds} & \multicolumn{2}{c|}{Slow Down Factor} \\ \cline{2-7}
Size & Limit on & Actual & Max Binary & {\bf TenG Time}  & ScaG / TenG & ScaBaG / TenG\\ 
        &  Size      & Size     & Rank          & {\bf Normalized} & Time Ratio & Time Ratio \\ \hline \hline
 \multirow{8}{*}{$1$}
  &  \multirow{2}{*}{$5$-$10$~M}
  &  $6,992,414$  &  $ 19.7$  &  $66.6$  &  $11.5$  &  $47.0$ \\
  &  & $1,830,909$  &  $0.7$     &  $20.6$  &  $4.6$    &  $20.3$  \\ \cline{2-7}
  &  \multirow{2}{*}{$15$-$20$~M}
  &   $17,979,799$  &  $21.2$  &  $34.3$  &  $22.2$  &  $82.1$ \\
  & &  $1,391,918$  &  $0.5$    &  $3.0$    &  $7.1$    &  $23.7$ \\ \cline{2-7}
  &  \multirow{2}{*}{$25$-$30$~M}
      &  $26,540,961$  &  $21.6$  &  $20.8$  &  $38.4$  &  $137.3$ \\
  & &  $1,154,660$  &  $ 0.5$  &  $4.6$  &  $14.6$  &  $56.1$ \\ \cline{2-7}
  &  \multirow{2}{*}{$35$-$40$~M}
      &  $37,058,914$  &  $21.8$  &  $16.0$  &  $50.3$  &  $177.2$ \\
  &  &  $1,349,479$  &  $0.4$  &  $3.5$  &  $32.1$  &  $128.7$ \\  \hline \hline
   \multirow{8}{*}{$10$}
  &  \multirow{2}{*}{$5$-$10$~M}
  &      $8,157,025$  &  $20.0$  &  $7.8$  &  $112.3$  &  $38.0$ \\
  &  &  $1,599,408$  &  $0.5$  &  $2.3$  &  $45.9$  &  $24.8$ \\ \cline{2-7}
  &  \multirow{2}{*}{$15$-$20$~M}
      &  $17,504,179$  &  $20.7$  &  $4.6$  &  $148.0$  &  $54.2$ \\
  &  &  $1,482,496$  &  $0.5$  &  $1.3$  &  $ 64.9$  &  $33.3$ \\ \cline{2-7}
  &  \multirow{2}{*}{$25$-$30$~M}
      &  $ 27,728,478$  &  $21.7$  &  $4.5$  &  $209.7$  &  $60.1$ \\
  &  &  $2,029,237$  &    $0.9$     &  $1.3$  &  $51.0$  &  $17.0$ \\ \cline{2-7}
  &  \multirow{2}{*}{$35$-$40$~M}
      &  $37,850,485$  &  $22.1$  &  $4.0$  &  $244.0$  &  $70.0$ \\
  &  &  $1,547,389$  &  $0.6$  &  $ 1.1$  &  $95.4$  &  $26.3$ \\ \hline \hline
   \multirow{8}{*}{$20$}
  &  \multirow{2}{*}{$5$-$10$~M}
      &  $6,506,125$  &  $19.6$  &  $4.9$  &  $135.3$  &  $26.9$ \\
  & &  $860,631$      &  $0.7$  &  $1.9$  &  $42.2$  &  $11.1$ \\ \cline{2-7}
  &  \multirow{2}{*}{$15$-$20$~M}
      &  $17,766,240$  &  $20.7$  &  $3.1$  &  $251.5$  &  $39.9$ \\
  &  &  $1,209,040$  &  $0.5$  &  $1.3$     &  $123.7$  &  $15.2$ \\ \cline{2-7}
  &  \multirow{2}{*}{$25$-$30$~M}
      &  $27,762,672$  &  $ 21.7$  &  $3.1$  &  $271.9$  &  $46.0$ \\
  &  &  $1,148,761$  &  $0.5$  &  $1.1$  &  $92.4$  &  $17.7$ \\ \cline{2-7}
  &  \multirow{2}{*}{$35$-$40$~M}
      &  $37,620,063$  &  $22.1$  &  $3.0$  &  $287.5$  &  $44.3$ \\
  &  &  $1,416,214$  &    $0.3$  &   $1.2$   &  $118.9$  &  $19.7$ \\ \hline
\end{tabular}
}
\label{tab:tensor-time}
\end{table}

We next present two experiments on randomly generated Bayesian networks.
The first experiment further evaluates our proposed algorithm for exploiting
functional dependencies. The second experiment reinforces our motivation for working
with dense representations of factors and the corresponding tensor-based 
implementations.\footnote{The experiments of this section were run on a server with dual Intel(R) 
Xeon E5-2670 CPUs running at 2.60GHz and 256GB RAM.
}

We generated Bayesian networks by starting with a linear order of nodes \(V_1, \ldots, V_n\) and
a maximum number of parents per node \(k\). For each node \(V_i\), we randomly determined
a number of parents \(\leq k\) and chose the parents randomly from the set \(V_1, \ldots, V_{i-1}\).
We then selected a fixed percentage \(f\) of non-root nodes and gave them functional CPTs,
where each node had cardinality two or three.

In the first experiment, we considered networks with different number of nodes \(n\), maximum
number of parents \(k\) and percentage of functional nodes \(f\). For each combination, we generated
\(10\) networks, computed a binary jointree and averaged the 
size of largest cluster. We then applied our algorithm for exploiting functional dependencies
and obtained a jointree with shrunk clusters and separators while also noting the size of largest
cluster.

Table~\ref{tab:reduce-tw} depicts our results, where we report the size of a largest cluster in terms of
its {\em binary rank:} log2 of its instantiations count.  As can
be seen from Table~\ref{tab:reduce-tw}, our algorithm leads to substantial reductions in binary
rank, where the reduction increases as the fraction of functional nodes increases. Recall that
our algorithm includes two heuristics: one for deciding how to replicate functional CPTs
when building a jointree and another corresponding to the choice on
Line~\ref{ln:choice} in Figure~\ref{fig:alg}. Table~\ref{tab:reduce-tw} provides some evidence
on the efficacy of these heuristics beyond the rectangle and digits case studies we discussed
earlier.

The second experiment compares classical and tensor-based implementations of ACs.
In a classical implementation, the AC is represented as a directed acyclic graph where root
nodes correspond to scalars and other nodes correspond to scalar arithmetic operations;
see Figure~\ref{fig:ac}.
We will call this the {\em scalar graph} representation.
In a tensor-based implementation, the AC is represented using a {\em tensor graph} where root 
nodes correspond to tensors (i.e., factors) and other nodes correspond to tensor operations 
(i.e., factor operations) as discussed in Section~\ref{sec:tensors}. The main benefit of
a tensor-based implementation is that tensor operations can be parallelized on CPUs
and GPUs (for example, \numpy\ and tools such as TensorFlow leverage Single 
Instruction Multiple Data (SIMD) parallelism on CPUs).\footnote{\url{https://en.wikipedia.org/wiki/SIMD}}

Before we present the results of this experiment, we need to discuss the notion
of a {\em batch} which is a set of AC input vectors. When learning the parameters of an
AC using gradient descent, the dataset or a subset of it can be viewed as a batch
so we would be interested in evaluating the AC on a batch. A scalar graph would need
to be evaluated on each input vector in a batch separately. However, when representing 
the AC as a tensor graph we can treat the batch as a tensor. This allows us to evaluate
the AC on a batch to yield a batch of marginals, which creates more opportunities for parallelism.

There is middle grounds though: a scalar graph with a batch that we shall call the
{\em scalar-batch graph.} This is a tensor graph except that
each tensor has two dimensions only: a batch dimension and a scalar dimension. For example,
if the batch has size \(b\), then a tensor will have shape \((b,1)\). In a scalar-batch
graph, each tensor is a set of scalars, one for each member of the batch (AC input vector). 

Scalar-batch graphs can be used in situations where a full tensor graph cannot be
used. This includes situations where the AC is compiled using techniques such as 
knowledge compilation, which produce ACs that cannot be cast in terms of tensor
operations. A scalar-batch graph can be used in this case to offer an opportunity
for parallelism, even if limited, especially when training the AC from data.

\begin{figure}[tb]
\small
\centering
\alglanguage{pseudocode}
\begin{algorithmic}[1]
\Procedure{evaluate\_scalar\_graph}{graph\_nodes}
\For{\(n\) in graph\_nodes:}
\State c1, c2 = n.child1, n.child2
\If{n.type == 'add'}
\State n.value = c1.value + c2.value
\ElsIf{n.type == 'mul'}
\State n.value = c1.value * c2.value
\ElsIf{n.type == 'div'}
\State n.value = c1.value / c2.value
\EndIf
\EndFor
\EndProcedure
\end{algorithmic}
\caption{Evaluating a scalar graph representation of an AC. The graph nodes
are topologically sorted so the children of a node are evaluated before the node
is evaluated. The evaluation
of a scalar-batch representation is similar except that node values
are \numpy\ ndarrays and \(+, *, /\) are ndarray (tensor) operations. \label{fig:scalarG}}
\end{figure}

Table~\ref{tab:tensor-time} compares the three discussed AC representations in terms
of their evaluation time, while varying the batch size and AC size. The tensor 
graph implementation is the one we discussed
in Section~\ref{sec:tensors} using TensorFlow. The scalar graph implementation
uses a Python list to store the DAG nodes (parents before children) and then
uses the pseudocode in Figure~\ref{fig:scalarG} to evaluate the DAG. We extract
the DAG from the tensor graph where each DAG node corresponds to a tensor entry. 
The scalar-batch graph is represented similarly to the scalar graph except that
members of the list are \numpy\ ndarrays of shape \((b,1)\) instead of scalars (we
found \numpy\ to be more efficient than TensorFlow for this task). 
The evaluation time for both scalar graphs and scalar-batch graphs are therefore 
based on benchmarking the code in Figure~\ref{fig:scalarG} (we only measure
the time of arithmetic operations, excluding setting evidence on root nodes
and other overhead). 

The networks in Table~\ref{tab:tensor-time} were generated randomly as in the previous experiment,
with \(100\) nodes and a maximum of \(5\) parents per node.
For each limit on the AC size, we kept generating Bayesian networks randomly
until we found \(10\) networks whose compilations yielded tensor graphs
within the given size limit.
The {\em tensor graph normalized time}
in Table~\ref{tab:tensor-time} is the total time to evaluate the graph divided
by the batch size, then divided again by the size of the graph over \(1000,000\).
Normalized time is then the average time for evaluating one million AC nodes (tensor entries)
and is meant to give a sense of speed independent of the batch and AC size.

We now have a number of observations on Table~\ref{tab:tensor-time}.
The tensor graph is faster than the scalar and scalar-batch
graphs in all cases and sometimes by two orders of magnitude. This can
be seen in the last two columns of Table~\ref{tab:tensor-time} which report
the evaluation times (whole batch) of scalar and scalar-batch graphs over the
evaluation time of tensor graph.
The gap between tensor and scalar graphs increases with the batch size
and with AC size as this means more opportunities to exploit parallelism
on two fronts that the scalar graph cannot take advantage of. 
The gap between the tensor and scalar-batch graphs
increases with AC size, but decreases with batch size. Increasing the AC 
size correlates with increasing the size of tensors (at least the largest
one in the fourth column) which creates more opportunities for exploiting parallelism that the scalar-batch
graph cannot exploit. However, increasing the batch size can be
exploited by both the tensor and scalar-batch graphs, therefore narrowing
the gap (\numpy\ appears to be exploiting the batch more effectively than TensorFlow).
The scalar graph is faster than the scalar-batch graph when the batch
size is \(1\), but otherwise is slower. This is to be expected as there
is no need for the extra overhead of \numpy\ ndarrays in this case.
We finally emphasize the absolute evaluation times for the tensor
graph, which amount to a few milliseconds per one million AC nodes (normalized time)
when the batch and AC size are large enough.
\section{Conclusion}
\label{sec:conclusion}
We presented new results on the algorithm of variable elimination that exploit
functional dependencies using dense factors, allowing one to 
benefit from tensor-based technologies for more efficient inference and learning.
We also presented case studies that show the promise of proposed techniques. 
In contrast to earlier approaches, the proposed one does not dependent on the identity 
of functional dependencies, only that they are present. This has further applications to exact inference
(exploiting inferred functional dependencies) and to approximate inference (treating CPTs with
extreme probabilities as functional CPTs) which we plan to pursue in future work.

\ack I wish to thank members
of the Automated Reasoning Group at UCLA who provided valuable motivation and feedback:
Arthur Choi, Yizou Chen, Haiying Huang and Jason Shen.
This work has been partially supported by grants from NSF IIS-1910317, 
ONR N00014-18-1-2561 and DARPA N66001-17-2-4032. 

\bibliography{bib/references,bib/references2,bib/refs,bib/refsnn}

\end{document}

%% file: stdmacros.tex
\newtheorem{thm}{Theorem}

\newtheorem{corollary}{Corollary}

\newcommand\shrink[1]{}

\def\n(#1){\bar{#1}}

\def\pr{{\it Pr}}

\def\C{{\bf C}}

\def\p{{\bf p}}

\def\U{{\bf U}}
\def\u{{\bf u}}

\def\X{{\bf X}}
\def\x{{\bf x}}
\def\Y{{\bf Y}}
\def\y{{\bf y}}
\def\Z{{\bf Z}}
\def\z{{\bf z}}

\def\eql(#1,#2){{#1\!\!=\!#2}}

\def\eql(#1,#2){{#1\!=\!#2}}

\def\clap#1{\hbox to 0pt{\hss#1\hss}}